\documentclass[11pt,a4paper]{article}
\usepackage[round]{natbib}
\usepackage{amssymb}
\usepackage{amsthm}
\usepackage{dsfont}
\usepackage{amsmath}
\usepackage{graphicx}
\usepackage{xcolor}    
\usepackage{pifont}    
\usepackage{algorithm,algpseudocode}

\newtheorem{proposition}{Proposition}

\usepackage{xcolor}
\usepackage{hyperref}

\definecolor{citegreen}{rgb}{0,0.6,0}  

\usepackage[capitalise,nameinlink]{cleveref}

\usepackage{kantlipsum}

\begin{document}
\normalsize
\title{Online Conformal Probabilistic Numerics via\\[0.3cm]
Adaptive Edge-Cloud Offloading}

\author{
Qiushuo Hou\thanks{Zhejiang University}\\
\textit{qshou@zju.edu.cn} \\
\and
Sangwoo Park\thanks{Imperial College London}$^{\dagger\ddagger}$ \\
\textit{s.park@imperial.ac.uk} \\
\and
Matteo Zecchin\thanks{King’s College London} \\
\textit{matteo.1.zecchin@kcl.ac.uk}
\and
Yunlong Cai$^{*}$ \\
\textit{ylcai@zju.edu.cn} \\
\and
Guanding Yu$^{*}$ \\
\textit{yuguanding@zju.edu.cn} \\
\and
Osvaldo Simeone$^{\ddagger}$\\
\textit{osvaldo.simeone@kcl.ac.uk}
}

\date{}
\maketitle
\begin{abstract}
Consider an edge computing setting in which a user submits queries for the solution of a linear system to an edge processor, which is subject to time-varying computing availability. The edge processor applies a probabilistic linear solver (PLS) so as to be able to respond to the user's query within the allotted time and computing budget. Feedback to the user is in the form of a set of plausible solutions. Due to model misspecification, the highest-probability-density (HPD) set obtained via a direct application of PLS does not come with coverage guarantees with respect to the true solution of the linear system. This work introduces a new method to calibrate the HPD sets produced by PLS with the aim of  guaranteeing long-term coverage requirements. The proposed method, referred to as online conformal prediction-PLS (OCP-PLS), assumes sporadic feedback from cloud to edge. This enables the online calibration of uncertainty thresholds via online conformal prediction (OCP), an online optimization method previously studied in the context of prediction models. The validity of OCP-PLS is verified via experiments that bring insights into trade-offs between coverage, prediction set size, and cloud usage.
\end{abstract}

\section{Introduction}
In modern hierarchical computing architectures encompassing edge and cloud servers, the efficient management of computational resources is critical to balance performance with latency and communication overhead.  Edge servers process data near the user, providing low-latency, but possibly inaccurate, responses due to limited computational resources. In contrast, cloud servers offer more powerful processing capabilities, but at the cost of introducing communication delays. As an example, in industrial Internet-of-Things (IoT) settings, sensor data processing may leverage both real-time decision-making at the edge and more complex data analytics in the cloud, while accounting for the resulting inherent trade-offs between response speed and accuracy \citep{hong2019multi}.

Linear systems serve as the cornerstone of virtually all numerical computation. These systems -- in the form of the equation $Ax=b$ with unknown vector $x$ -- arise in contexts including convex optimization \citep{boyd2004convex}, Kalman filtering for state estimation \citep{Kalman_application}, and finite element analysis for computational fluid dynamics \citep{PDE_application}. A timely and efficient solution to such systems is often critical for real-time decision-making.

As illustrated in \cref{fig:overall}, we consider an edge computing scenario in which, at each round $t=1,2,...$, a user submits a query for the solution of a linear system $A_t x_t =b_t$ to an edge processor, which is subject to time-varying computing power availability.  The user seeks to obtain timely information about the solution $x_t^*=A_t^{-1}b_t$. However, a direct evaluation of the solution requires cubic computational complexity with respect to the matrix dimension, which 
 may be infeasible within the given latency and computing budgets    \citep{wenger2020probabilistic,cockayne2019bayesian}. 

Iterative solvers provide an ideal solution to adapt to the available computing budget, as they refine the solution sequentially until the computing budget runs out. \emph{Probabilistic numerics}  (PN) \citep{larkin1972gaussian} treats numerical problems, such as solving linear systems,  as a form of statistical inference. This provides a natural way to quantify the uncertainty associated with iterative numerical solutions obtained under limited computing power.  

Specifically, the \emph{probabilistic linear solver} (PLS) \citep{cockayne2019bayesian,hennig2015probabilistic,bartels2019probabilistic, reid2020bayescg, pfortner2024computation} typically treat the solution of a linear system as a Bayesian inference problem, whereby each iteration provides an updated posterior distribution for the solution $x^*_t$. The posterior distribution reflects the current knowledge of the true solution given the available computing budget. Using this posterior distribution, PLS can construct a \emph{highest-probability-density (HPD) set} $\mathcal{C}_t$ with the aim of covering the true solution $x^*_t$ with a pre-determined target probability level.

In the considered system illustrated in Fig. 1, at each round $t$, the edge processor applies PLS so as to be able to respond to the user's query within the allotted time and computing budgets. Feedback to the user is in the form of a HPD set $\mathcal{C}_t$. However, in the presence of \emph{model misspecification}, the HPD sets obtained using a direct application of PLS do not come with coverage guarantees with respect to the true solution $x^*_t$ of the linear system \citep{cockayne2019bayesian,bartels2019probabilistic,wenger2020probabilistic,hennig2015probabilistic}.

\begin{figure}[t]
    \centering
    \includegraphics[width=160pt]{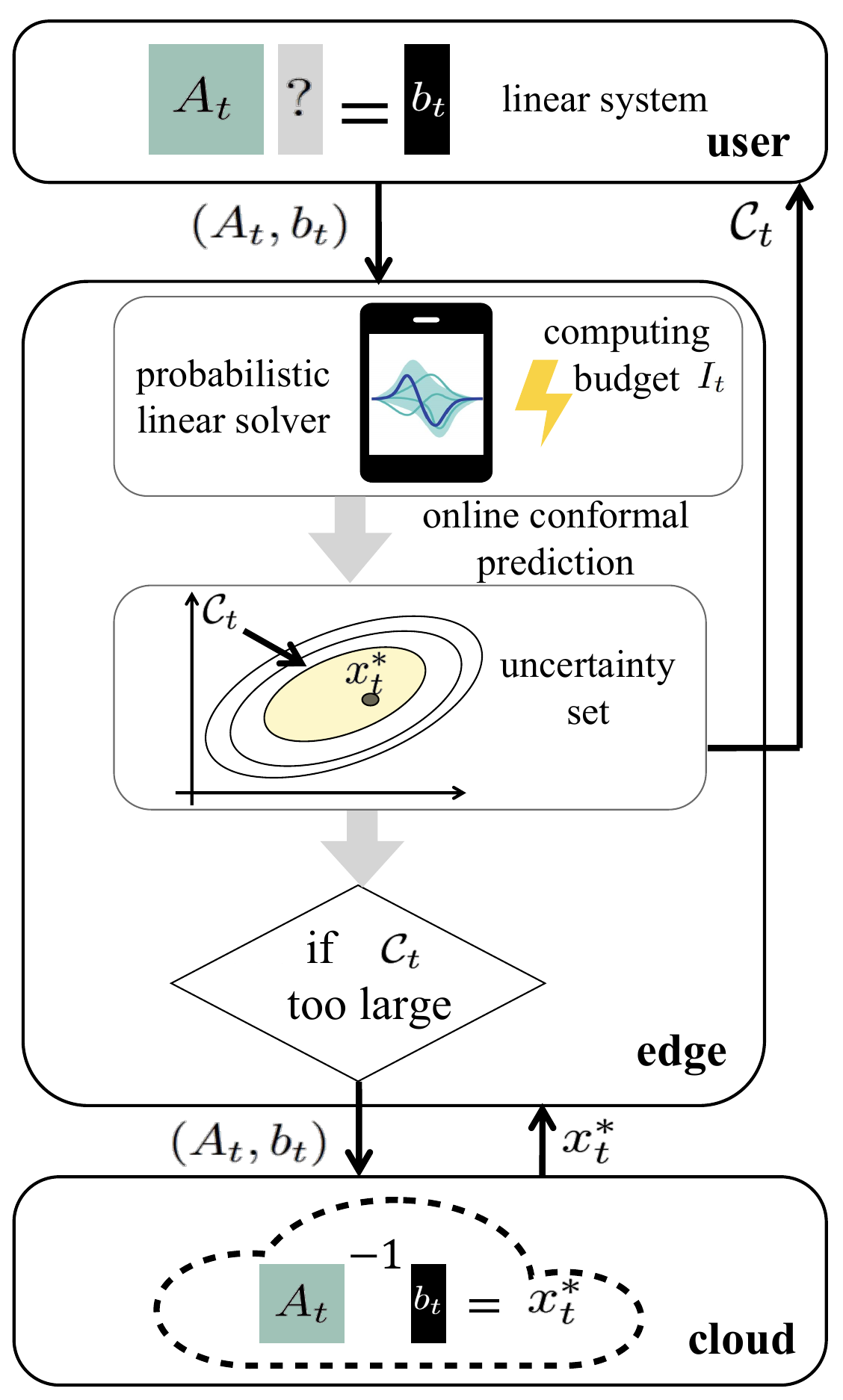}
    \caption{At each round $t$, a user submits a linear system defined by the pair $(A_t, b_t)$ to an edge device. Given the available computing budget $I_t$, the edge device employs a probabilistic linear solver (PLS), obtaining a highest-probability-density (HPD) set $\mathcal{C}_t$ for the true solution $x^*_t = A_t^{-1}b_t$. The set $\mathcal{C}_t$ is returned in a timely fashion to the user. The proposed method, OCP-PLS ensures long-term coverage guarantees (\ref{eq:coverage_guarantee}) for the HPD sets $\mathcal{C}_t$, addressing model misspecification in PLS. To this end, OCP-PLS allows for sporadic communication between cloud and edge.}
    \label{fig:overall}
\end{figure}

Model misspecification arises from the choice of a distribution that does not reflect the properties of the true solution $x^*_t$. For example, one typically assumes a Gaussian prior, but the solution $x^*_t$ may be sparse. Misspecification may also arise due to simplifications done in the evaluation of the likelihood in order to obtain Gaussian posteriors, as is the case with BayesCG \citep{hegde2024calibrated,reid2020bayescg, wenger2020probabilistic,cockayne2019bayesian,cockayne2022testing,cockayne2019bayesian_PN}, which may reinforce incorrect prior assumptions. Further work on the impact of calibration on PLS includes reference \citep{reid2020bayescg}, which demonstrated that BayesCG search directions yield a slightly optimistic HPD set under the Krylov prior, and references \citep{hegde2024calibrated, cockayne2021probabilistic}, which explored the use of probabilistic stationary iterative methods rather than Bayesian framework.

This work introduces a new method to calibrate the HPD sets produced by PLS with the aim of guaranteeing long-term coverage requirements. In practice, we wish to ensure that, on average over time, the HPD set $\mathcal{C}_t$ returned by the edge to the user contains the true solution $x^*_t$ with a user-defined \emph{coverage} rate $1-\alpha$.  

To ensure this condition, we assume that, \emph{sporadically}, the edge processor can submit the current linear system defined by the pair $(A_t,b_t)$ to a cloud processor. This offloading to the cloud is done after responding to a user's query in order not to affect the latency experienced by the user. When submitting the job to the cloud, the edge processor receives the true solution $x^*_t$. This information can be used to monitor the current coverage rate, making it possible to calibrate the prediction sets $\mathcal{C}_t$ towards meeting the required target rate $1-\alpha$. 

The proposed method, referred to as \emph{online conformal prediction-PLS} (OCP-PLS), integrates for the first time PLS with \emph{online conformal prediction} (OCP),  an online optimization method previously studied in the context of predictive models \citep{gibbs2021adaptive, angelopoulos2024online}. The theoretical validity of OCP-PLS is verified via experiments that bring insights into trade-offs between coverage, prediction set size, and cloud usage.

The structure of the paper is as follows. \cref{sec:setting_prob} introduces the problem of conformal PLS. The necessary background on OCP is reviewed in \cref{sec:background}. \cref{sec:method} presents the proposed approaches for constructing OCP-PLS. Simulation results are summarized in \cref{sec:simulation}. Finally, \cref{sec:conclusion} concludes the paper.

\section{Setting and Problem Formulation}\label{sec:setting_prob}
In this section, we define the problem studied in this paper.

\subsection{Setting and Problem Definition}
As illustrated in \cref{fig:overall}, we study a hierarchical computing architecture encompassing edge and cloud processors. At each round $t$, a user submits a linear system problem defined by the pair $(A_t, b_t)$, consisting of an $n_t\times n_t$ Hermitian matrix $A_t$ and an $n_t\times 1$ vector $b_t$, to the edge device. The user is interested in obtaining information about the solution 
\begin{equation}\label{eq:linear_system}
    x^*_t = A_t^{-1} b_t.
\end{equation}
This solution is supposed to exist and to be unique, requiring matrix $A_t$ to be invertible. Problems of this type are common in a wide variety of applications, such as convex optimization using Newton's method requiring Hessian matrix inversion \citep{boyd2004convex}, Kalman filtering for state estimation \citep{Kalman_application}, and partial differential equations in computational fluid dynamics via Galerkin’s method \citep{PDE_application}.

\begin{figure}[t]
    \includegraphics[width=280pt]{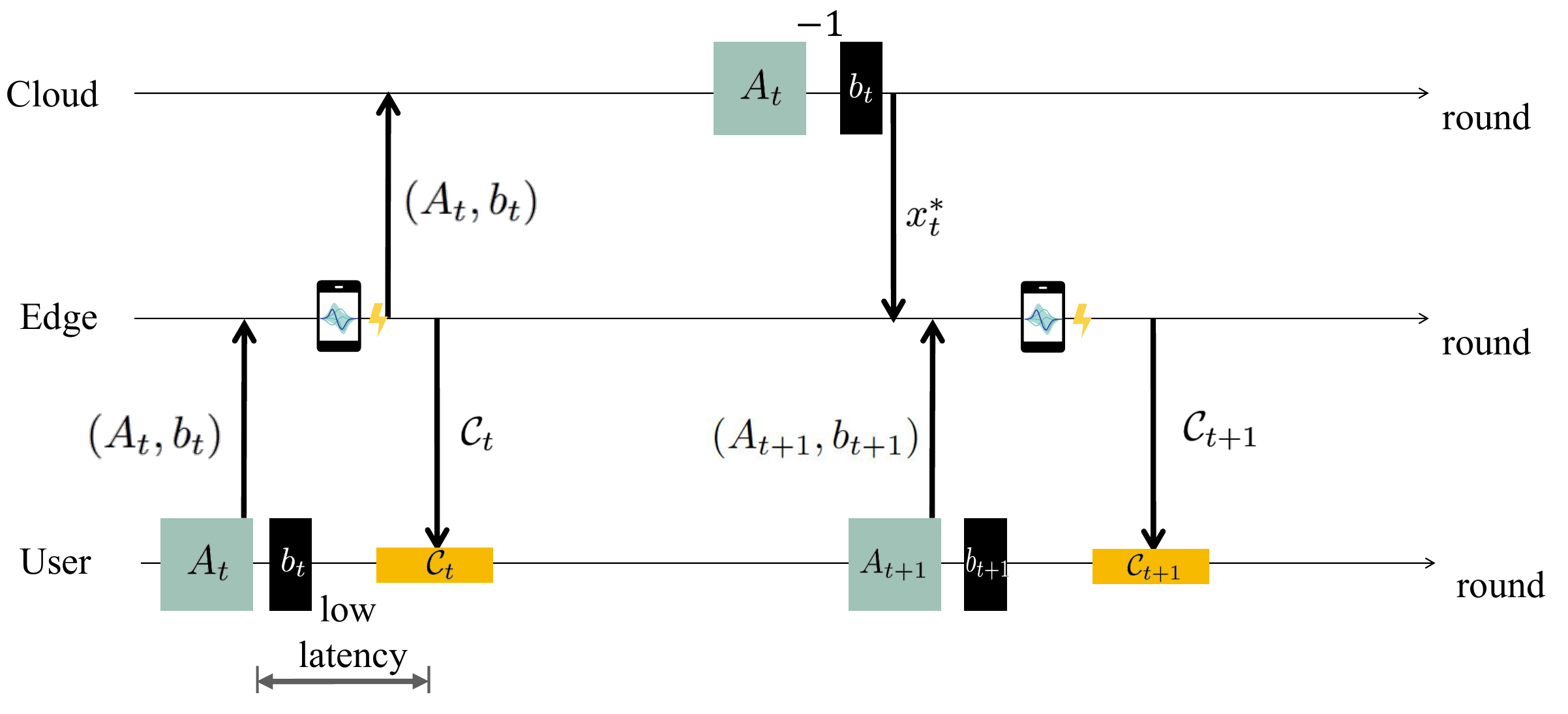}
    \centering
    \caption{Timeline of messages exchanged across cloud, edge, and user.}
    \label{fig:time_step}
\end{figure}

At each round $t$, the edge has available a computing budget $I_t$. Accounting for this limited budget, the edge device employs PLS to obtain an HPD set $\mathcal{C}_t$ for the true solution (\ref{eq:linear_system}) \citep{cockayne2019bayesian,bartels2019probabilistic,wenger2020probabilistic}. This set is returned to the user in response to its query. Thanks to processing at the edge, the user's requests experience low latency (see \cref{fig:time_step}).

This paper addresses the problem of \emph{calibrating} the prediction set $\mathcal{C}_t$ so as to ensure that the sequence of HPD sets $\mathcal{C}_t$ returned to the user satisfies the following long-term coverage condition
\begin{equation}\label{eq:coverage_guarantee}
    \left|\frac{1}{T}\sum_{t=1}^T \Pr\{x_t^* \in \mathcal{C}_{t}\} -(1-\alpha)\right| = o(1),
\end{equation}
where $\alpha \in [0, 1]$ is a pre-determined target miscoverage level, and $o(1)$ denotes a quantity that tends to zero as $T\rightarrow \infty$. The probability (\ref{eq:coverage_guarantee}) is evaluated over the randomness associated with edge-cloud communication, as it will be described in \cref{sec:method}. By (\ref{eq:coverage_guarantee}), a fraction approximately equal to $1-\alpha$ of the prediction sets contains the true solution $x^*_t$.

As shown in \cref{fig:time_step}, after responding to the user, the edge processor analyzes the prediction set $\mathcal{C}_t$, and decides whether to submit the problem $(A_t,b_t)$ for validation to the cloud. If the edge processor communicates the pair $(A_t, b_t)$ to the cloud, the cloud evaluates the true solution $x^*_t$ and returns it to the edge. This allows the edge to monitor the condition (\ref{eq:coverage_guarantee}), and to calibrate future sets $C_{t'}$ with $t'>t$. However, cloud-based validation is expensive, and the edge aims at keeping the number of problems submitted to the cloud as low as possible. As shown in \cref{fig:time_step}, we will assume first that feedback from the cloud, if requested, is received prior to the next query $(A_{t+1},b_{t+1})$ by the user. This assumption will be alleviated in the appendix.

\subsection{Probabilistic Linear Solver}\label{sub:PLS}
As previously mentioned, the edge processor addresses the problem $(A_t, b_t)$ within the computing budget $I_t$ using PLS. Accordingly, PLS treats the solution of a linear system as a Bayesian inference problem. This enables the quantification of the uncertainty associated with numerical solutions obtained within a limited computing budget \citep{cockayne2019bayesian,bartels2019probabilistic,wenger2020probabilistic}. 

Given a budget of $I_t$ iterations, PLS operates sequentially by considering at each iteration $i=1,\ldots, I_t$ a \emph{search direction} $s_i$, yielding an effective observation $y_i = s_i^Tb$. Given a Gaussian prior distribution $x^*\sim\mathcal{N}(\mu_{0}, \Sigma_{0})$ on the solution $x^* = A_t^{-1}b_t$ with a full-rank diagonal matrix $\Sigma_0$, PLS returns the Gaussian posterior
\begin{equation}\label{eq:posterio}
    p(x_t|y_1,\dots,y_{I_t}; A_t,b_t) = \mathcal{N}(x_t|\mu_t,\Sigma_t)
\end{equation}
obtained from the observations $y_1,\dots,y_{I_t}$. In (\ref{eq:posterio}), the $d\times 1$ mean vector $\mu_{t}$ and the $d \times d$ covariance matrix $\Sigma_{t}$ are evaluated as \citep{cockayne2019bayesian,bartels2019probabilistic}
\begin{equation}\label{eq:mean and variance}
    \begin{aligned}
\mu_{t} &= \mu_{0} + \Sigma_{0} A_tS_{t}(S_{t}^{\top}A_t\Sigma_{0} A_tS_{t})^{-1}S_{t}^{\top}(b - A_t \mu_0),\\
\Sigma_{t} &= \Sigma_{0} - \Sigma_{0} A_tS_{t}(S_t^{\top}A_t\Sigma_{0} A_tS_{t})^{-1}S_{t}^{\top}A_t\Sigma_{0},
\end{aligned}
\end{equation}
with the $d\times I_t$ matrix 
\begin{equation}\label{eq:search_direction}
    S_{t}=[s_1, ..., s_{I_t}]
\end{equation}
collecting the $I_t$ search directions. Note that the covariance matrix $\Sigma_t$ has rank
\begin{equation}
    r_t \leq n_t-I_t,
\end{equation}
with equality achieved if the directions $S_t$ are suitably chosen \citep{wenger2022posterior}. Note also that, due to the propagation of floating point error, the rank $r_t$ may be practically larger than $n_t-I_t$. This may be the case, e.g., when using BayesCG as the search directions and the number of iterations $I_t$ is close to the dimension $n_t$ \citep{cockayne2019bayesian}.

Given the posterior distribution (\ref{eq:posterio}), an $(1-\alpha)$-HPD set for the true solution $x^*_t$ can be obtained as
\begin{align} \label{eq:credible}
    \mathcal{C}_t = &\mathop{\rm{argmin}}\limits_{\mathcal{C}\subseteq \mathbb{R}^d}\ |\mathcal{C}| \nonumber\\&\text{ s.t. } \int_{x_t \in \mathcal{C}} p(x_t|A_t,b_t) \mathrm{d}x_t \geq 1-\alpha,
\end{align}
where we wrote $p(x_t|A_t, b_t)$ in lieu of $p(x_t|y_1,\dots,y_{I_t}; A_t,b_t)$ to simplify the notation. 

The HPD set (\ref{eq:credible}) covers the true solution $x^*_t$ in (\ref{eq:linear_system}) with probability no smaller than $1-\alpha$, i.e.,
\begin{align}\label{eq:calibration}
    \Pr[x^*_t \in \mathcal{C}_t]\geq 1-\alpha,
\end{align}
only if the posterior distribution $p(x_t|A_t,b_t)$ accurately quantifies the uncertainty associated with estimating the solution $x^*_t$ given the observation $y_1,\dots,y_{I_t}$. Accordingly, the posterior distribution $p(x_t|A_t,b_t)$ is regarded as \emph{well calibrated} when the coverage probability in (\ref{eq:calibration}) equals the target level $1-\alpha$.
 
However, the PLS posterior (\ref{eq:posterio}) is known to be generally poorly calibrated because of model misspecification \citep{cockayne2019bayesian,bartels2019probabilistic,wenger2020probabilistic,hennig2015probabilistic}. Therefore, there is a need to develop model selection to satisfy the coverage requirement (\ref{eq:coverage_guarantee}). This is the main objective of this work.

\section{Background}\label{sec:background}
This section reviews the necessary background on OCP \citep{gibbs2021adaptive, zhao2024conformalized}.

\subsection{Online Conformal Prediction}\label{subsec:OCP}

Given an arbitrary sequence of input-output pairs $(x_t,y_t)\in\mathcal{X}\times \mathcal{Y}$, for $t = 1, 2, \ldots$, OCP aims to construct a prediction set $\mathcal{C}_t(x_t)$ on the output space $\mathcal{Y}$ that satisfies long-term coverage guarantees. OCP builds on a scoring function $s: \mathcal{X}\times \mathcal{Y} \rightarrow \mathbb{R}$ that measures the extent to which the output $y$ is mismatched to the input $x$. A smaller value $s(x,y)$ indicates that $y$ is a better fit for input $x$. The set $\mathcal{C}_t(x_t)$ is obtained by choosing all the candidate output values $y \in \mathcal{Y}$ with sufficiently small scores $s(x,y)$, i.e.,
\begin{align} \label{eq:ocp_set}
    \mathcal{C}_t(x_t) = \{ y \in \mathcal{Y}: s(x_t, y) \leq \lambda_t \},
\end{align}
where $\lambda_t$ is a threshold.

OCP updates the threshold $\lambda_t$ in (\ref{eq:ocp_set}) based on the input-output pairs $(x_1,y_1),\ldots,(x_{t-1}, y_{t-1})$ with the aim of ensuring the long-term coverage condition 
\begin{align} \label{eq:basic_ocp}
    \left|\frac{1}{T} \sum_{t=1}^T \mathds{1}\{y_t \in \mathcal{C}_t(x_t)\}-(1-\alpha)\right| \leq\frac{C}{T},
\end{align}
where $C$ is a constant independent of $T$. This result is obtained by updating the threshold as
\begin{align} \label{eq:update_rule_ocp}
    \lambda_{t+1} = \lambda_{t} + \gamma\big(\mathds{1}\{y_{t} \not\in \mathcal{C}_{t}(x_{t})\} - \alpha\big),
\end{align}
where $\gamma>0$ is a step size. This update rule intuitively increases the threshold when the prediction set $\mathcal{C}_{t}(x_t)$ fails to cover the true outcome $y_t$, making future sets more conservative, while decreasing it when coverage is achieved, making future sets more efficient. 

The validity of the condition (\ref{eq:basic_ocp}) for the prediction sets (\ref{eq:ocp_set}) obtained via the OCP update rule (\ref{eq:update_rule_ocp}) were shown in \cite{gibbs2021adaptive} using a telescoping argument, obtaining
\begin{equation}\label{eq:constant}
    C = \max \frac{|\lambda_{T+1}-\lambda_1|}{\gamma}.
\end{equation}
If the score function $s(x,y)$ is bounded within the interval $[m,M]$  with $M> m$ for any $(x,y)\in\mathcal{X}\times\mathcal{Y}$, then the threshold $\lambda_t$ can be also proved to be bounded as $\lambda_t \in [m-\gamma \alpha, M+\gamma(1-\alpha)]$, where $\lambda_1$ is selected within the interval $[m,M]$ \citep{gibbs2021adaptive,feldman2022achieving}. Accordingly, the quantity (\ref{eq:constant}) is given as

\begin{equation}
    C= \frac{M-m+\gamma}{\gamma},
\end{equation}
which does not depend on $T$.

\subsection{Intermittent Online Conformal Prediction}\label{sub:I-OCP}
The OCP update (\ref{eq:update_rule_ocp}) requires \emph{dense} feedback, as it assumes the availability of the true output $y_t$ at every time step $t=1,2,...$ in (\ref{eq:update_rule_ocp}). Intermittent OCP (I-OCP) \citep{zhao2024conformalized} alleviates this assumption by leveraging inverse probability weighting, also known as the Horvitz--Thompson estimator, which is used in statistical inference with missing data \citep{IPW}.

Specifically, I-OCP assumes that the output $y_t$ is available with probability $p_t$. The availability of the output is thus described by a Bernoulli random variable $\text{obs}_t\sim \text{Bern}(p_t)$, which indicates the availability of the output value $y_t$ if $\text{obs}_t=1$, and its absence when $\text{obs}_t=0$. The variables $\text{obs}_t$ are independent over the time index $t$. The update rule applied by I-OCP is given by 
\begin{align} \label{eq:update_rule_iocp}
    \lambda_{t+1} = \lambda_{t} + \frac{\gamma}{p_{t}}\big(\mathds{1}\{y_{t} \not\in \mathcal{C}_{t}(x_{t})\} - \alpha\big)\cdot \text{obs}_{t}.
\end{align}
Thus, I-OCP amplifies the update of the threshold to an extent that is inversely
proportional to the availability probability $p_t$. Accordingly, when feedback is less likely, i.e., when $p_t$ is smaller, the effective step size is increased to compensate for possible missing data.

It can be proved that I-OCP satisfies the long-term expected coverage property \citep{zhao2024conformalized}
\begin{align} \label{eq:guarantee_iocp}
    \left|\frac{1}{T} \sum_{t=1}^T \Pr\big[y_t \in \mathcal{C}_t(x_t)\big]-(1-\alpha) \right|\leq\frac{C}{T}
\end{align}
with $C$ in (\ref{eq:constant}), where the probability is evaluated over the randomness of the Bernoulli random variables $\text{obs}_1,\ldots, \text{obs}_T$.

In a manner similar to OCP, as long as the score function $s(x,y)$ is bounded in the interval $[m,M]$ for any $(x,y)\in\mathcal{X}\times\mathcal{Y}$, it can be shown that the quantity $C$ in (\ref{eq:guarantee_iocp}) is also bounded. Specifically, denote as $p^\text{min}_t$ the lowest value that can be attained by the probability $p_t$. Note that the probability $p_t$ can generally depend on the previous history $(x_1,y_1),\ldots,(x_{t-1}, y_{t-1})$. Then, the threshold $\lambda_t$ is bounded as
\begin{equation}
\small
    \lambda_t \in \bigg[ m-\gamma\frac{\alpha}{\min_{t'=1,...,t-1}{{p}^\text{min}_{t'}}}, M+\gamma\frac{1-\alpha}{\min_{t'=1,...,t-1}{{p}^\text{min}_{t'}}} \bigg]
\end{equation}
as long as one sets $\lambda_1\in[m,M]$. Accordingly, we have

\begin{align}
    C = \frac{M-m}{\gamma} + \frac{1}{\bar{p}_T},
\end{align}
for any 
\begin{equation}
    \bar{p}_T = \min_{t=1,...,T}{p}^\text{min}_{t}.
\end{equation}
This quantity does not depend on $T$ if a strictly positive $\bar{p}$ exists with
\begin{equation}\label{eq:probability_p}
    \bar{p}_T\geq\bar{p}
\end{equation}
for all times $T$.

\section{Online Conformal Probabilistic Linear Solver}\label{sec:method}

In this section, we present OCP-PLS, a novel framework for edge inference via PN that satisfies the long-term quality-of-service requirement (\ref{eq:coverage_guarantee}) by integrating PLS (see in \cref{sub:PLS}) with I-OCP (see \cref{sec:background}).

\subsection{Score Function}
As discussed in \cref{subsec:OCP}, OCP requires the definition of a score function for any value of the target variable. In the setting described in \cref{sec:setting_prob}, the target variable at each time $t$ is the true solution $x^*_t$ in (\ref{eq:linear_system}). Given the posterior distribution (\ref{eq:posterio})-(\ref{eq:mean and variance}) produced by PLS, we set the score for any candidate solution $x_t$ to be proportional to the density of the scaled degenerate multivariate Gaussian distribution $\mathcal{N}(x_t|\mu_t,\Sigma_t)$. Let $\Sigma_t = V_t \Lambda_t V_t^T$ be the eigendecomposition of the rank-$r_t$ covariance matrix $\Sigma_t$, where $V_t$ is a unitary $n_t \times r_t$ matrix and $\Lambda_t$ is an $r_t \times r_t$ diagonal matrix. The score can be explicitly written as
\begin{equation}\label{eq:score_function}
\resizebox{1\columnwidth}{!}{$
s(x_t; A_t, b_t)
= \left\{
\begin{array}{ll}
\exp(-(x_t - \mu_t)^T V_t\Lambda_t^{-1} V_t^T (x_t - \mu_t)) & \text{if } (x_t - \mu_t) \in \text{range}(\Sigma_t), \\[1em]
0 & \text{otherwise}.
\end{array}
\right.
$}
\end{equation}
By doing this, the score function is bounded between $0$ and $1$. 

Since matrix $\Sigma_t$ in (\ref{eq:mean and variance}) can be expressed as a sequence of rank-$1$ updates to the prior diagonal matrix $\Sigma_0$, its eigendecomposition can be efficiently updated using Given's rotations \cite[Section 12.5]{golub2013matrix}. This results in a per-iteration complexity of $\mathcal{O}(n_t^2)$, and an overall complexity of $\mathcal{O}(I_t n_t^2)$ after $I_t$ updates. Note that the score (\ref{eq:score_function}) depends on the specific sequence of search directions (\ref{eq:search_direction}), although we do not make this dependence explicit in the notation $s(x_t; A_t, b_t)$.

Equipped with the scoring function (\ref{eq:score_function}), OCP-PLS produces a prediction set of the form
\begin{equation}\label{eq:prediction_set}
\mathcal{C}_{t} = \{x_t \in \mathbb{R}^d:s(x_t; A_t, b_t) \geq \lambda_t\},
\end{equation}
with a suitably designed threshold $\lambda_t$.

Two observations are in order about the OCP-PLS set (\ref{eq:prediction_set}). First, the PLS set (\ref{eq:credible}) can be seen to take the same form as (\ref{eq:prediction_set}). However, PLS selects the threshold $\lambda_t$ by imposing the constraint in (\ref{eq:credible}), which satisfies the requirement (\ref{eq:coverage_guarantee}) only if the model is well specified. In contrast, as discussed in \cref{subsec:feedback}, OCP-PLS sets the threshold $\lambda_t$ by following the principles of OCP (see \cref{sec:background}).

\subsection{Cloud-to-Edge Feedback and Threshold Selection}\label{subsec:feedback}
As illustrated in \cref{fig:overall} and \cref{fig:time_step}, in order to set the threshold $\lambda_t$ in (\ref{eq:prediction_set}), OCP-PLS leverages sporadic feedback from the cloud to the edge. Specifically,
we propose to submit a query $(A_t, b_t)$ to the cloud with a probability $p_t$ that increases with the \emph{volume radius} of the set $\mathcal{C}_t$, i.e., with $|\mathcal{C}_t|^{1/r_t}$ where
\begin{align}
    |\mathcal{C}_t| = \frac{\pi^{r_t/2}}{\Gamma\big(r_t/2 + 1\big)} \cdot |\Lambda_t|^{1/2} \cdot (-\log\lambda_t)^{r_t/2},
\end{align}
with $\Gamma(\cdot)$ being the Gamma function. In this way, cloud-to-edge feedback is more likely to be requested at time instants $t$ characterized by a larger uncertainty, hence making cloud feedback more valuable for calibration. Conversely, when the HPD set is small enough, we reduce the probability of requesting feedback, thereby saving cloud computing resources. 

While any normalized increasing function of the set size $|\mathcal{C}_t|$ can be used to design $p_t$, we adopt the sigmoid function $\sigma(x) = \left(1+\text{exp}(-x)\right)^{-1}$ and set as
\begin{equation}\label{eq:p_t}
    p_t = \sigma\left(\frac{\log|\mathcal{C}_t|}{r_t}-\theta\right),
\end{equation}
which acts as a soft threshold on the log-volume radius. Lower values of the threshold $\theta$ in (\ref{eq:p_t}) increase the probability of cloud-to-edge feedback, yielding a smaller set $\mathcal{C}_t$ at the cost of higher usage of cloud resources.

Leveraging the feedback from the cloud, OCP-PLS targets the requirement (\ref{eq:coverage_guarantee}) by updating the threshold according to the rule
\begin{equation}\label{eq:update_SF}
\lambda_{t+1} = \lambda_{t} - \frac{\gamma}{p_{t}}( \text{err}_{t} -\alpha)\cdot\text{obs}_{t},
\end{equation}
where the error at round $t$ is defined as
\begin{align}
    \text{err}_t = \mathds{1}\{ x^*_{t} \notin \mathcal{C}_{t} \}.
\end{align}
The rule (\ref{eq:update_SF}) is aligned with the I-OCP update (\ref{eq:update_rule_iocp}). As formalized next, this ensures the validity of the coverage condition (\ref{eq:coverage_guarantee}). The overall OCP-PLS procedure is summarized in \Cref{alg:1}. 
\begin{algorithm}
    \begin{algorithmic}[1]
        \State set hyperparameters $\lambda_1\in[0,1]$, $\gamma$, $\theta$
 \Comment{initialization}
        \For{$t = 1,2,\dots,T$} 
        \State apply PLS with $I_t$ iterations to evaluate $\mu_t$ and $\Sigma_t$ in (\ref{eq:mean and variance})
        \State evaluate the probability $p_t$ in (\ref{eq:p_t})
        \State draw $\text{obs}_t\sim \text{Bern}(p_t)$
        \State evaluate the HPD set $\mathcal{C}_t$ in (\ref{eq:prediction_set})
        \If{$\text{obs}_t = 1$}
            \State obtain $x^*_t$ from the cloud
            \State $\text{err}_t = \mathds{1}\{s(x^*_t;A_t, b_t)< \lambda_t\}$
            \State $\lambda_{t+1} = \lambda_{t} -\frac{\gamma}{p_{t}}(\text{err}_{t}-\alpha)\text{obs}_{t}$ \Comment{update $\lambda_t$}
        \Else
            \State $\lambda_{t+1} = \lambda_{t}$
        \EndIf
        \EndFor
    \end{algorithmic}
    \caption{OCP-PLS}
    \label{alg:1}
\end{algorithm}

\subsection{Theoretical Coverage Guarantee}\label{sec:coverage_theory}
OCP-PLS ensures condition (\ref{eq:coverage_guarantee}) in the following way.
\begin{proposition}\label{proposition}
    Assume that there exists a strictly positive probability $\bar{p}$ satisfying the inequality (\ref{eq:probability_p}) for all times $T$. Then, OCP-PLS satisfies the long-term coverage property
    
    \begin{equation}
        \left|\frac{1}{T} \sum_{t=1}^T \Pr\big[x^*_t \in \mathcal{C}_t\big] -(1-\alpha)\right| \leq \frac{C}{T},
    \end{equation}
    with constant
    \begin{equation}
        C = \frac{1}{\gamma} + \frac{1}{\bar{p}}.
    \end{equation}
\end{proposition}

\begin{proof}
    This result follows directly from \cref{sub:I-OCP} by noting that $s(x_t; A_t, b_t)$ in (\ref{eq:score_function}) is bounded in $[0,1]$.
\end{proof}

\section{Experiments}\label{sec:simulation}
This section evaluates the performance of the proposed OCP-PLS scheme via numerical experiments.

\subsection{Setup}
At each round $t$, we randomly and independently generate a Hermitian matrix $A_t$ as $A_t = Q_t\Lambda_t Q_t$, where $n_t\times n_t$ orthogonal matrix $Q_t$ follows the Haar distribution and the diagonal matrix $\Lambda_t$ has diagonal elements drawn from a Gamma distribution with shape parameter $10$, scale parameter $1$, and location parameter $0$. The dimension $n_t$ is drawn i.i.d. and uniformly from the interval $[500, 1000]$. Furthermore, each element of vector $b_t$ is sampled i.i.d. from the standard multivariate Gaussian distribution. We use the prior $x \sim \mathcal{N}(0, I)$, and BayesCG for the search directions \citep{cockayne2019bayesian_PN}.

We set the target miscoverage rate $\alpha$ as $0.1$. The step size $\gamma$ and initial threshold $\lambda_1$ are set as $0.05$ and $0.99$, respectively. For the hyperparameter $\theta$, we use grid search to find a suitable one as $-3.5$. 

We consider two different scenarios in terms of edge processing power availability:
\begin{itemize}
    \item \textbf{Constant edge computing budget}: In this case, the computing budget is set as a fixed fraction of the problem size as $I_t = 0.1n_t$.
    \item \textbf{Time-varying edge computing budget}: In this more challenging scenario, the number of allowed iterations starts at $I_t = 0.1n_t$ for $t = 1$ to $t = 1500$, then it decreases to $I_t = 0.005n_t$ for $t = 1500$ to $t = 3500$, and finally it increases to $I_t = 0.15n_t$ for $t = 3500$ to $t = 5000$.
\end{itemize}

\subsection{Baselines}
We consider the following baselines:
\begin{itemize}
    \item \textbf{PLS}: Conventional PLS directly adopts the $(1-\alpha)$-HPD set as in (\ref{eq:credible}).
    \item \textbf{OCP-PLS with full cloud-to-edge feedback}: This corresponds to the proposed OCP-PLS method with probability $p_t = 1$ of acquiring feedback from cloud to edge.
\end{itemize}

\subsection{Performance Metrics}
As performance measures, we report empirical coverage, empirical HPD set size, and empirical cloud-to-edge feedback rate evaluated on a sequence $\{A_t, b_t\}_{t=1}^T$ of duration $T=5000$. The empirical coverage at round $t$ measures the fraction of times in which the exact solution $x^*_\tau$ was included in the corresponding HPD set $\mathcal{C}_t$ from time $\tau=1$ until time $\tau=t$, i.e.,
\begin{equation}\label{eq:coverage}
    \text{empirical coverage} = \frac{1}{t}\sum_{\tau=1}^t\mathds{1}\{x^*_\tau\in\mathcal{C}_\tau\}.
\end{equation}
The size of the HDP set is measured by the empirical volume radius as
\begin{equation}
    \text{empirical volume radius} = \frac{1}{t}\sum_{\tau=1}^t|\mathcal{C}_\tau|^{1/r_\tau}.
\end{equation}

Finally, the empirical cumulative cloud-to-edge feedback rate is defined as
\begin{align}\label{eq:feedback}
\text{empirical cumulative} & \text{ cloud-to-edge feedback rate} = \sum_{\tau=1}^t \text{obs}_\tau.
\end{align}

We average the probability (\ref{eq:coverage})-(\ref{eq:feedback}) over $10$ independent experiments.

\subsection{Results}

In this subsection, we evaluate coverage, average set size, and cumulative cloud-to-edge feedback rate starting from the case with a constant edge computing budget.

\cref{fig:iid}(a) demonstrates that both OCP-PLS and OCP-PLS with full cloud-to-edge feedback guarantee the long-term coverage condition (\ref{eq:coverage_guarantee}), thereby validating \cref{proposition}. In contrast, confirming prior art \citep{hegde2024calibrated,wenger2020probabilistic}, PLS produces an overly conservative HPD set that always covers the true solution. PLS achieves coverage of $1$ by yielding much larger HPD sets than both OCP-PLS and OCP-PLS with full cloud-to-edge feedback, as shown in\cref{fig:iid}(b). Furthermore, as illustrated in (\ref{eq:coverage_guarantee}), OCP-PLS can reduce the cloud-to-edge feedback rate by approximately $40\%$ as compared to OCP-PLS with full cloud-to-edge feedback, while producing sets with similar sizes.

\begin{figure*}[t]
    \includegraphics[width=350pt]{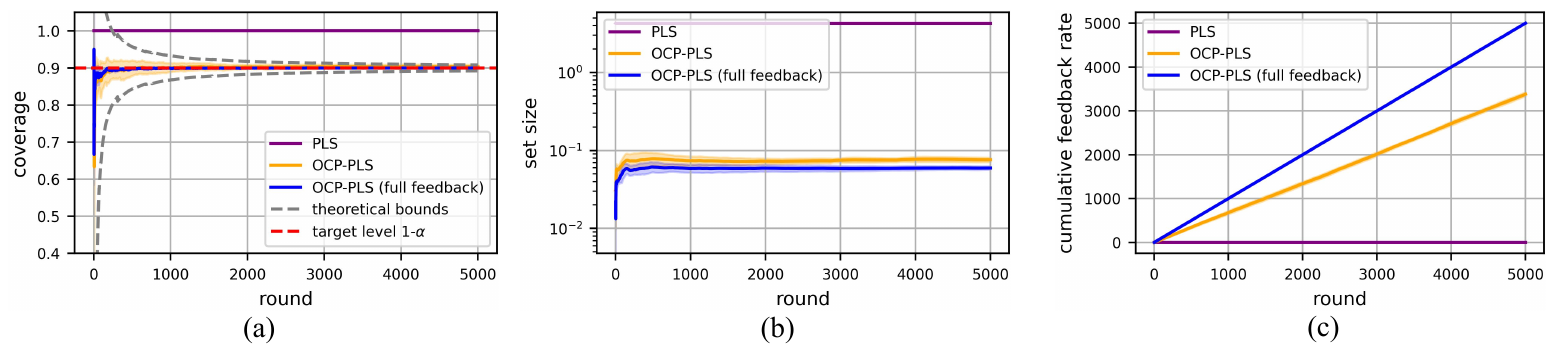}
    \centering
    \caption{Coverage (a), set size (b), and cumulative cloud-to-edge feedback rate (c) for PLS, OCP-PLS, and OCP-PLS with full cloud-to-edge feedback for target coverage level $1-\alpha = 0.9$ (red dashed line in (a)). The gray dashed line in (a) corresponds to the threshold bound (\ref{eq:guarantee_iocp}).}
    \label{fig:iid}
\end{figure*}

\begin{figure*}[t]
    \includegraphics[width=350pt]{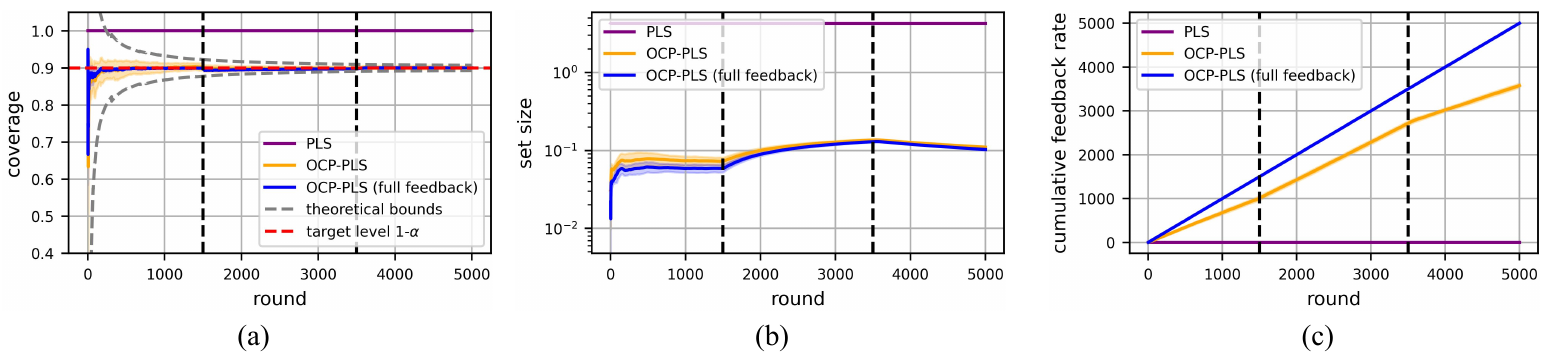}
    \centering
    \caption{Coverage (a), set size (b), and cumulative cloud-to-edge feedback rate (c) for PLS, OCP-PLS, and OCP-PLS with full cloud-to-edge feedback for target coverage level $1-\alpha = 0.9$ (red dashed line in (a)) in a setting with time-varying edge computing budget. The gray dashed line in (a) corresponds to the threshold bound (\ref{eq:guarantee_iocp}).}
    \label{fig:non_iid}
\end{figure*}

Consider now the setting with a time-varying computing budget at the edge. \cref{fig:non_iid} marks the times at which the computing budget changes with vertical dashed lines. \cref{fig:non_iid}(a) verifies that both OCP-PLS and OCP-PLS with full cloud-to-edge feedback maintain the guaranteed long-term coverage condition proved in \cref{proposition} even under time-varying computing budgets. As in the previous example, PLS provides perfect coverage at the cost of excessively large HPD sets. In contrast, as shown in \cref{fig:non_iid}(b), OCP-PLS responds naturally to changes in edge computing budget, as the set sizes increase during periods of reduced computing budget ($t=1500 $ to $t=3500$), while decreasing when the computing budget is higher ($t=3500$ to $t=5000$). Throughout these variations, OCP-PLS consistently maintains a significantly lower cloud-to-edge feedback rate than OCP-PLS with full cloud-to-edge feedback, as shown in \cref{fig:non_iid}(c), adapting the feedback rate to the varying computing budgets. This demonstrates OCP-PLS's ability to balance reliability and communication efficiency under dynamic edge computing budgets.

\section{Conclusion}\label{sec:conclusion}
An important challenge in edge-cloud computing architectures is to provide users with timely and reliable responses under limited computational budgets. This work has made steps towards addressing this challenge by leveraging probabilistic numerics (PN) and online conformal prediction (OCP), an online calibration method previously studied in the context of predictive models. The proposed methodology, referred to as OCP-PLS calibrates the HPD sets produced by probabilistic linear solvers (PLS), ensuring provable coverage guarantees for the true solution. OCP-PLS supports an adaptive mechanism for the use of cloud resources that reduces the communication overhead between edge and cloud, as well as the computational load of the cloud processors. Experiments for tasks requiring the sequential solution of linear systems validate the theoretical reliability, as well as the efficiency, of OCP-PLS. Future work may investigate the generalization of the proposed approach to other numerical problems beyond linear systems \citep{wenger2022posterior,pfortner2022physics,hegde2024calibrated}.

\section*{Appendix A: Extension to Delayed Cloud-to-Edge Feedback}\label{sec: extension}
In this appendix, we extend the analysis of OCP-PLS to a setting with delayed cloud-to-edge feedback. Specifically, we assume that, given a request to the cloud at the $t$-th time instant, the corresponding cloud-to-edge feedback may not arrive before the $(t+1)$-th request from the user. We denote as $d_t$ the corresponding feedback delay, i.e., the feedback arrives right before serving the $(t+1+d_t)$-th time instant and after the $(t+d_t)$-th time instant, where $d_t\geq 1$ is an integer. Note that the case $d_t = 0$ corresponds to the setting considered in \cref{sec:setting_prob}.

A direct extension of the threshold update rule (\ref{eq:update_SF}) is obtained by updating the threshold whenever the edge receives feedback from the cloud. This can be written as 
\begin{equation}\label{eq:update_delay_feedback}
   \lambda_{t+1} = \lambda_t - \gamma\sum_{t'\in \mathcal{T}_t}\frac{\text{obs}_{t'}}{p_{t'}}\big(\mathds{1}\{x^*_{t'} \notin \mathcal{C}_{t'}\}-\alpha\big),
\end{equation}
where the set $\mathcal{T}_t$ is  defined as
\begin{equation}
   \mathcal{T}_t = \{t': t'+d_{t'} = t\}.
\end{equation}

We will now impose mild assumptions on the delays $d_t$ so as to ensure the validity of OCP-PLS. Specifically, we assume the inequality
\begin{equation}\label{eq:order}
    t+d_t \leq t'+d_{t'}
\end{equation}
for all $t<t'$, so that the feedback signals are received in the same order as the requests are sent to the cloud. We also assume that the cardinality of the set $\mathcal{T}_t$ is bounded as
\begin{equation}\label{eq:constrain}
    |\mathcal{T}_t|\leq T^{\max}
\end{equation}
for some finite integer $T^{\max}$.

\begin{proposition} \label{prop:delay}
Assume that there exists a strictly positive probability $\bar{p}$ satisfying the inequality (\ref{eq:probability_p}) for all times $T+T^{\max}$. Then, for any fixed sequence of delays $d_1,d_2,\ldots$, that ensure assumptions (\ref{eq:order})-(\ref{eq:constrain}), the OCP-PLS update rule (\ref{eq:update_delay_feedback}) guarantees the long-term expected coverage property
\begin{equation}\label{eq:assumption3}
   \left|\frac{1}{T}\sum_{t=1}^T\Pr[x^*_t\in\mathcal{C}_t] - (1-\alpha)\right|\leq\frac{C}{T},
\end{equation}
with constant 
   \begin{equation}
       C = T^{\max}-1+\frac{1}{\gamma}+\frac{T^{\max}}{\bar{p}}.
   \end{equation}
\end{proposition}
Note that \cref{prop:delay} reduces to  \cref{proposition} in \cref{sec:coverage_theory} by setting $T^{\max}=1$ and $d_t=0$, i.e., when there is no delay in the cloud-edge feedback.

\begin{proof}
   Taking the expectation on the both sides of equation (\ref{eq:update_delay_feedback}) and using a telescoping argument for $t=1,...,T+d_T+1$, we get
   \begin{equation}
       \begin{aligned}
   \mathbb{E}[\lambda_{T+d_T+1}] &= \lambda_1 - \gamma\sum_{t'=1}^{T+d_T}\sum_{t\in\mathcal{T}_{t'}}\big(\Pr[x_t^* \notin \mathcal{C}_t] - \alpha  \big),\\
   &=\lambda_1-\gamma\sum_{t=1}^{T'} \big(\Pr[x_t^* \notin \mathcal{C}_t] - \alpha  \big),
   \end{aligned}
   \end{equation}
   
   where $T' = \max \{ t \in \mathcal{T}_{T+d_T} \}$ is the index of the most recent clould-to-edge feedback used to update $\lambda_{T+d_T+1}$.

The expected coverage up to time $T'$ can be obtained as 
    \begin{align} \label{eq:app_1}
       & \frac{1}{T'} \sum_{t=1}^{T'} \Pr[x_t^* \in \mathcal{C}_t] = 1- \alpha - \frac{\mathbb{E}[\lambda_{T+d_T+1}]-\lambda_1}{T' \gamma}.
    \end{align}
Based on (\ref{eq:app_1}), we express the expected coverage up to time $T$ is follows
\begin{equation}
\footnotesize
    \begin{aligned}
        &\frac{1}{T}\sum_{t=1}^T\Pr[x^*_t\in\mathcal{C}_t] = \frac{T'}{T}\frac{1}{T'}\sum_{t=1}^{T'}\Pr[x^*_t\in\mathcal{C}_t] - \frac{1}{T}\sum_{t=T+1}^{T'}\Pr[x^*_t\in\mathcal{C}_t],\\
        & \overset{(a)}{=} \frac{T'}{T}(1-\alpha)-\frac{\mathbb{E}[\lambda_{T+d_T+1}]-\lambda_1}{T\gamma} - \frac{1}{T}\sum_{t=T+1}^{T'}\Pr[x^*_t\in\mathcal{C}_t],\\
        & = (1-\alpha) + \frac{(T'-T)\cdot(1-\alpha)-\sum_{t=T+1}^{T'}\Pr[x^*_t\in\mathcal{C}_t]}{T} \\&- \frac{\mathbb{E}[\lambda_{T+d_T+1}]-\lambda_1}{T\gamma},
    \end{aligned}
\end{equation}
where $(a)$ uses (\ref{eq:app_1}). The coverage gap can be bounded as 
\begin{equation}\label{final_bound}
\scriptsize
    \begin{aligned}
        &\left|\frac{1}{T}\sum_{t=1}^T\Pr[x^*_t\in\mathcal{C}_t] - (1-\alpha)\right|\\
        &\leq \frac{\left|(T'-T)\cdot(1-\alpha)-\sum_{t=T+1}^{T'}\Pr[x^*_t\in\mathcal{C}_t]\right|}{T} + \frac{\left|\mathbb{E}[\lambda_{T+d_T+1}]-\lambda_1\right|}{T\gamma},
    \end{aligned}
\end{equation}
with the first term on the right-hand side appearing due to the presence of a delay in the cloud-edge feedback. From Assumption (\ref{eq:assumption3}), it follows that $T'-T\leq T^{\max}-1$, and therefore we can bound this term as 
\begin{equation}\label{first_bound}
    \left|(T'-T)\cdot(1-\alpha)-\sum_{t=T+1}^{T'}\Pr[x^*_t\in\mathcal{C}_t]\right|\leq T^{\max}-1.
\end{equation}

Using a proof technique similar to the one used without delayed feedback in \cref{sec:background}, given a score function $s(x,y)$ that is bounded in the interval $[0,1]$ for any $(x,y)\in\mathcal{X}\times\mathcal{Y}$, the threshold $\lambda_t$ obtained via (\ref{eq:update_delay_feedback}) satisfies
    \begin{equation}
        \lambda_t\in\bigg[-\gamma\frac{T^{\max}(1-\alpha)}{\min_{t' = 1,\ldots,t-1}p_{t'}^{\text{min}}}, 1+\gamma\frac{T^{\max}\alpha}{\min_{t' = 1,\ldots,t-1}p_{t'}^{\text{min}}}\bigg],
    \end{equation}
    as long as one selects $\lambda_1\in[0,1]$. Accordingly, we have
    \begin{equation}\label{second bound}
    \begin{aligned}
        \frac{\left|\mathbb{E}[\lambda_{T+d_T+1}]-\lambda_1\right|}{\gamma} \leq \frac{1}{\gamma} + \frac{T^{\max}}{\bar{p}_{T+T^{\max}}},
    \end{aligned}
    \end{equation}
    where
    \begin{equation}
    \bar{p}_{T+T^{\max}} = \min_{t'=1,...,T+T^{\max}}{p}^\text{min}_{t}.
    \end{equation}
    This quantity does not depend on $T$ if there exists $\bar{p}>0$ such that
    \begin{equation}\label{eq:probability_p_appendix}
        \bar{p}_{T+T^{\max}}\geq\bar{p}
    \end{equation}
    for all times $T+T^{\max}$.

    Combining (\ref{first_bound}) and (\ref{second bound}), we obtain 
    \begin{equation}
       C = T^{\max}-1+\frac{1}{\gamma}+\frac{T^{\max}}{\bar{p}_{T+T^{\max}}} \leq T^{\max}-1+\frac{1}{\gamma}+\frac{T^{\max}}{\bar{p}},
    \end{equation}
    which concludes the proof. 
\end{proof}

\section*{Acknowledgements}
The work of M. Zecchin and O. Simeone is supported by the European Union’s Horizon Europe project CENTRIC (101096379). The work of O. Simeone is also supported by an Open Fellowship of the EPSRC (EP/W024101/1), and by the EPSRC project (EP/X011852/1).


\bibliography{references}

\end{document}